\documentclass[letterpaper]{article} 

\newif\ifextendedversion
\extendedversiontrue   

\usepackage{aaai2026}  
\usepackage{times}  
\usepackage{helvet}  
\usepackage{courier}  
\usepackage[hyphens]{url}  
\usepackage{graphicx} 
\usepackage{natbib}  
\usepackage{caption} 
\usepackage{latexsym}
\usepackage{amssymb}
\usepackage{amsmath}
\usepackage{amsthm}
\usepackage{booktabs}
\usepackage{enumitem}
\usepackage{graphicx}
\usepackage{color}
\usepackage{todonotes}
\usepackage{subfig}

\usepackage[ruled,vlined, linesnumbered, noend]{algorithm2e}
\SetInd{0.1em}{0.4em}
\usepackage{etoolbox}

\usepackage{circuitikz}
\usetikzlibrary{patterns} 

\usepackage{hhline}

\usepackage{tikz}
\usetikzlibrary{calc,tikzmark}
\usepackage{float}

\usepackage{soul} 
\usepackage{xcolor} 
\sethlcolor{gray!20}

\usepackage{graphicx,verbatimbox,caption,stackengine}

\newtheorem{theorem}{Theorem}
\newtheorem{lemma}[theorem]{Lemma}

\newtheorem{definition}{Definition}

\setitemize{itemsep=1pt,topsep=1pt,itemindent=10pt,leftmargin=0pt}
\setenumerate{itemsep=1pt,topsep=1pt,itemindent=10pt,leftmargin=0pt}

\newcommand{\BibTeX}{B\kern-.05em{\sc i\kern-.025em b}\kern-.08em\TeX}
\newcommand{\rivals}{\ensuremath{\widetilde{\mathcal{G}}}}

\newcommand{\myparagraph}[1]{\vskip2pt\noindent\textbf{#1.}}

\DeclareMathAlphabet{\mathpzc}{OT1}{pzc}{m}{it}

\SetKwComment{Comment}{$\triangleright$\ }{}
\SetCommentSty{itshape}
\SetKw{Continue}{continue}

\makeatletter
\patchcmd{\algocf@makecaption@ruled}{\hsize}{\textwidth}{}{} 
\patchcmd{\@algocf@start}{-1.5em}{0em}{}{} 
\makeatother

\usepackage{newfloat}
\usepackage{listings}
\DeclareCaptionStyle{ruled}{labelfont=normalfont,labelsep=colon,strut=off} 
\floatstyle{ruled}
\newfloat{listing}{tb}{lst}{}
\floatname{listing}{Listing}
\ifextendedversion
    \title{Interleaving Scheduling and Motion Planning with Incremental Learning of Symbolic Space-Time Motion Abstractions \\ (Extended Version)}
\else
    \title{Interleaving Scheduling and Motion Planning with Incremental Learning of Symbolic Space-Time Motion Abstractions}
\fi

\author {
    Elisa Tosello\textsuperscript{\rm 1}, 
    Arthur Bit-Monnot\textsuperscript{\rm 2}, 
    Davide Lusuardi\textsuperscript{\rm 1}, 
    Alessandro Valentini\textsuperscript{\rm 1},
    Andrea Micheli\textsuperscript{\rm 1}
}
\affiliations {
    \textsuperscript{\rm 1}Fondazione Bruno Kessler, Trento, Italy\\
    \textsuperscript{\rm 2}LAAS-CNRS, Université de Toulouse, CNRS, INSA, Toulouse, France\\
    etosello@fbk.eu, abitmonnot@laas.fr, lusuardi@fbk.eu, alvalentini@fbk.eu, amicheli@fbk.eu
}

\usepackage{bibentry}

\begin{document}

\maketitle

\begin{abstract}
Task and Motion Planning combines high-level task sequencing (\textit{what} to do) with low-level motion planning (\textit{how} to do it) to generate feasible, collision-free execution plans.  
However, in many real-world domains, such as automated warehouses, tasks are predefined, shifting the challenge to \textit{if}, \textit{when}, and \textit{how} to execute them safely and efficiently under resource, time and motion constraints.
In this paper, we formalize this as the Scheduling and Motion Planning problem for multi-object navigation in shared workspaces. 
We propose a novel solution framework that interleaves off-the-shelf schedulers and motion planners in an incremental learning loop. The scheduler generates candidate plans, while the motion planner checks feasibility and returns symbolic feedback, i.e., spatial conflicts and timing adjustments, to guide the scheduler towards motion-feasible solutions.
We validate our proposal on logistics and job-shop scheduling benchmarks augmented with motion tasks, using state-of-the-art schedulers and sampling-based motion planners. Our results show the effectiveness of our framework in generating valid plans under complex temporal and spatial constraints, where synchronized motion is critical.
\end{abstract}

\begin{links}
    \ifextendedversion
        \link{Code}{https://github.com/fbk-pso/tampest.git}
    \else
      \link{Code}{https://github.com/fbk-pso/tampest.git}
      \link{Extended version}{http://arxiv.org/abs/2603.10651}
    \fi
\end{links}

\section{Introduction}
Task and Motion Planning (TAMP) is the problem of combining high-level decision-making, i.e., deciding which tasks to perform, with low-level motion planning, i.e., ensuring that these tasks are carried out via physically feasible, collision-free, trajectories~\cite{Garrett2021, Dantam2020}. This integration is critical in domains where symbolic actions must be grounded in real-world geometry and dynamics, including robotics and automated manufacturing.
While traditional TAMP focuses on \textit{what} to do and \textit{how} to execute it, many real-world scenarios assume a predetermined set of tasks, shifting the challenge to \textit{if} and \textit{when} to perform them. This reframes the problem as scheduling under resource, precedence, and timing constraints. 
For example, in an automated warehouse, mobile robots must transport goods from storage to delivery stations. With tasks such as move, pick, and drop predetermined, the problem becomes (i) deciding the order and timing for each robot (\textit{scheduling}) and (ii) computing dynamically and kinematically feasible, collision-free trajectories to execute these tasks in the continuous physical environment (\textit{motion planning}) (see Figure~\ref{fig:example1}). The interplay between space and time is crucial: motion planning must ensure not only spatial feasibility but also precise temporal coordination among agents, which may need to wait, sequence, or synchronize their movements to safely share constrained regions (e.g., narrow passages) and to prevent conflicts or deadlocks. Unlike discrete path-finding abstractions, this requires reasoning directly in continuous configuration spaces with explicit kinodynamic constraints.
We refer to this integrated challenge as the Scheduling and Motion Planning (SAMP) problem.

\begin{figure}
    \centering
    \includegraphics[width=\linewidth]{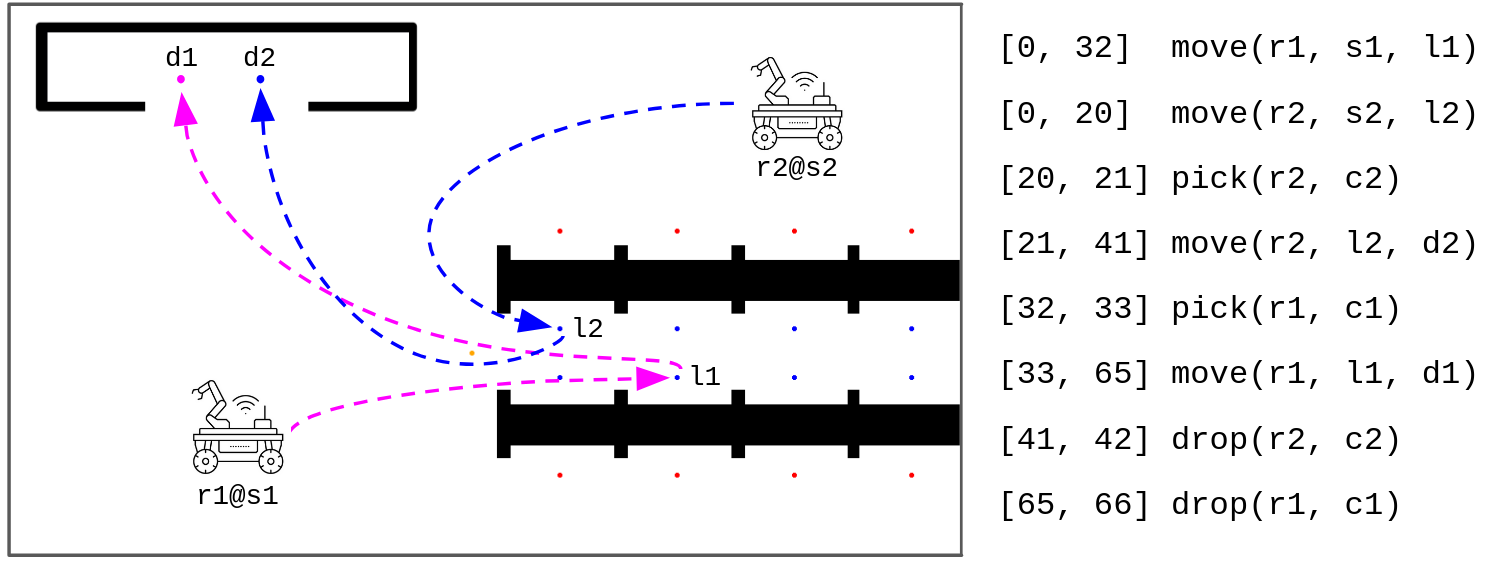}
    \caption{SAMP schedule of robots ($r_1$, $r_2$) executing move–pick–drop tasks. Intervals [$t_i$, $t_j$] denote start/end times. Robots move from ($s_1$, $s_2$) to pick ($c_1$, $c_2$) at ($l_1$, $l_2$) and deliver them to ($d_1$, $d_2$), parallelizing when possible.}
    \label{fig:example1}
\end{figure}

In this paper, we formally define the SAMP problem for multiple objects navigating in a shared workspace and propose a framework that addresses SAMP by interleaving  off-the-shelf schedulers and motion planners in an incremental learning loop of symbolic motion abstractions. The scheduler generates candidate schedules without considering the underlying motion. The motion planner, treated as a black-box, evaluates them accounting for the kinematics and dynamics of the objects involved, and returns either feasible trajectories or symbolic refinements to help the Scheduler finding a valid solution. Feedbacks include geometric refinements, highlighting spatial conflicts (i.e., unreachable goals and blocking obstacles), and temporal refinements, adjusting activity durations or requesting delays to enable feasible motion synchronization.
By incrementally learning such symbolic motion abstractions, our framework does not need to fully ground all constraints in advance, enabling better scalability in complex and dynamic domains.

We provide constraint formulations either via fluents (conditions and effects) or through precedence and resource constraints, enabling different schedulers (e.g., Aries~\cite{bitmonnot:hal-04174800} and OR-Tools~\cite{cpsatlp}) to be combined with (instrumented) motion planners (ST-RRT*~\cite{22-grothe-ICRA}) under optimal/non-optimal, with/without fluents settings. We evaluate these combinations on classical logistics and job-shop benchmarks~\cite{repec:eee:ejores:v:64:y:1993:i:2:p:278-285} augmented with navigation tasks. Results show that our framework produces valid, eventually optimal, synchronized plans under temporal and spatial constraints.

\section{Related Work}
Several studies address TAMP: from domain-specific solutions ~\cite{FFRob, Toussaint} to general frameworks ~\cite{dantam2016incremental, PDDLStream, Cashmore-2015, tosello2024}. While effective in combining symbolic and geometric planning, they neglect the temporal dimension, crucial when dealing with multi-agent scenarios. This gap led to works on temporal coordination, from motion-control strategies~\cite{Pecora_Andreasson_Mansouri_Petkov_2018} to optimal multi-agent planning~\cite{10102342} and Temporal Task and Motion Planning~\cite{2025-aaai}. However, they overlook cases where activities are predefined and the problem shifts toward scheduling, focusing on the temporal allocation and synchronization of known tasks rather than their dynamic generation.

This motivates Simultaneous Task and Motion Scheduling (STAAMS), which assigns and orders high-level actions while accounting for motion-level constraints. Although STAAMS combines Constraint Programming and Motion Planning, most existing approaches are tailored to specific domains, e.g., dual-arm manipulation~\cite{robotics10020057}, traffic coordination~\cite{10341755}, and assembly lines~\cite{10008036}. Some methods either rely on precomputed trajectories or perform online motion planning but use the scheduler primarily to enforce static constraints (e.g., collision avoidance or kinematic limits), often without temporal feedback~\cite{9197103} or tight scheduling-motion integration.
Such approaches could serve as a reasonable baseline for our method, but high customization and limited reproducibility hinder direct comparison.
We therefore start from benchmarks that may seem simpler (e.g., 2D navigation) but are generic and, importantly, require tightly interleaving scheduling with motion-level feasibility over time. 

To overcome the high computational cost of reasoning over continuous state spaces with explicit modeling of agent kinematics and dynamics, one may shift to Multi-Agent Path Finding (MAPF)~\cite{mapf-def}. However, classical MAPF assumes point-like agents on discrete spaces, neglecting geometry, kinematics, and dynamics. Extensions for large agents~\cite{10.1609/aaai.v33i01.33017627, 10.1609/aaai.v37i13.26812},  kinematic constraints~\cite{10.5555/3171837.3171975, 10.1609/aaai.v33i01.33017651} or temporal dependencies~\cite{10.1609/aaai.v39i22.34487} partially address these limitations, but rely on discretization, while continuous-time MAPF~\cite{ijcai2019p6} still omits full kinodynamic modeling.
In this paper, we remain focused on the continuous state space with kinodynamic constraints, leaving the combination with MAPF in discrete space as a new problem to be explored in the future.

\section{Problem Statement}
Consider a fleet of mobile robots moving products from shelves to a delivery station (see Figure~\ref{fig:example1}). Robots are movable objects whose configurations change during tasks, and the schedule defines when they move, pick items, or return, along with the trajectories and control laws enabling these actions. 
Activities can be optional, allowing the scheduler to skip them if they are desirable but non-essential (their inclusion improves plan quality) or to select among mutually exclusive alternatives (e.g., a delivery may be made to one of several possible locations).
Motions must be geometrically feasible, avoiding static (walls, shelves) and dynamic (other robots) obstacles, and temporally feasible, satisfying the timing scheduling constraints.
We call this problem SAMP and formalize it here, starting from the concept of Optional Scheduling (OS).

\begin{definition}\label{def:scheduling-with-fluents-and-effects}
An \textbf{Optional Scheduling} (OS) problem (with fluents and effects) is a tuple $\phi = \langle\mathcal{V}, \mathcal{A}, \mathcal{R}, \mathcal{C}, \textit{eff},  \textit{init}\rangle$: 
\begin{itemize}
    \item $\mathpzc{V}$ = $\{f_1, .., f_k\}$ is a finite set of fluents $f \in \mathpzc{V}$, each with a finite domain \textit{Dom(f)}.
    
    \item $\mathcal{A}$ is the set of mandatory and optional activities, where optional ones may be excluded. Each $a \in \mathcal{A}$ has a duration bound $[lb_a, ub_a]$, with $lb_a, ub_a \in \mathbb{N}^+$ being the lower and upper bounds, respectively. We denote a.present, a.start, and a.end the variables for presence, start, and end times of a.
     
    \item $\mathcal{R}$ is the set of available resources, with each resource $r \in  \mathcal{R}$ having availability $\lambda_r \in \mathbb{N}^+$.  

    \item $\mathcal{C} = \mathcal{C}_f \sqcup \mathcal{C}_t \sqcup \mathcal{C}_r$ is the set of constraints, divided into:
    \begin{itemize}
        \item $\mathcal{C}_f$ the set of fluent constraints of the form $([\kappa_1, \kappa_2], a, f = v)$, where $\kappa_i$ is an expression $a.start + k$ or $a.end-k$, $k \in \mathbb{N}$, with $a \in \mathcal{A}$, $f \in \mathcal{V}$, $v \in \textit{Dom(f)}$.
        
        \item $\mathcal{C}_t$ the set of constraints enforcing precedence relations and temporal ordering between activities. They are arbitrary Boolean combination of atoms of the form:
             \begin{itemize}
                 \item a.present for some $a \in \mathcal{A}$, where \textit{a}.present is True iff the activity is scheduled (i.e. appears in the solution);
                
                 \item $\kappa_1 - \kappa_2 \leq \Delta t$, with $\kappa_i \in \{\textit{a}_j.\text{start}, \textit{a}_j.\text{end}\}$ for  $a_j \in \mathcal{A} $, and $\Delta t \in \mathbb{Z}$ being the maximum delay between them. 
             \end{itemize}
        \item $\mathcal{C}_r$ the set of constraints on resource usage, where activity $a \in \mathcal{A}$ uses $\gamma_r^a$ units of resource $r$ over $[a.start, a.end]$.
        
         \end{itemize}

    \item $\textit{eff}: \mathcal{A} \rightarrow \mathcal{E}$ maps an activity to its timed effects on fluents. Each element of $\textit{eff}(a)$ is of the form $(\kappa, f := v)$ with $\kappa$ being either $a.start + k$ or $a.end-k$ with $k \in \mathbb{N}$, $f \in \mathcal{V}$ and $v \in \textit{Dom(f)}$; it indicates that at timing $\kappa$ (relative to $a$), fluent $f$ is assigned to value $v$ due to activity $a$.
    
    \item $\textit{init}$ is the initial fluent state, which assigns a value $\textit{init}(f) \in Dom(f)$ to each $f \in V$ at time 0.
\end{itemize}
\end{definition}

\noindent
The schedule solving an OS problem is defined as follows.   

\begin{definition}\label{def:schedule}  
A \textbf{schedule} $\rho$ solving $\phi$ is a tuple $\langle p, s, e\rangle$:
\begin{itemize}
    \item $p : \mathcal{A} \rightarrow \{\top, \bot\}$ indicates if an activity is present,
    \item $s : \mathcal{A} \rightarrow \mathbb{N}$ indicates the starting time of an activity, and
    \item $e : \mathcal{A} \rightarrow \mathbb{N}$ indicates the ending time.
\end{itemize}
\end{definition}

\noindent
We now define the semantics. In particular, for an expression $\kappa$ of the form $a_i.start + k$ (resp. $a_i.end - k$), its evaluation in $\rho$ is $\rho(\kappa) = s(a_i) + k$ (resp. $\rho(\kappa) = e(a_i) - k$). A schedule $\rho$ is \textbf{non-conflicting} if there exists no time $t \in \mathbb{N}$ and activities $a_1,a_2 \in \mathcal{A}$ with effects $(\kappa_1,f_1:=v_1) \in \textit{eff}(a_1), (\kappa_2,f_2:=v_2) \in \textit{eff}(a_2)$ such that $\rho(\kappa_1) = \rho(\kappa_2) \wedge f_1 = f_2$, i.e., two effects on the same fluent never overlap, as required by the PDDL semantics~\cite{pddl21}. 
To define the validity of a non-conflicting schedule, we first introduce a function tracking fluent changes over time.

\begin{definition}\label{def:eval-funct}
For a non-conflicting schedule $\rho = \langle p, s, e\rangle$, the \textbf{evaluation function} $\xi_\rho: \mathcal{V} \times \mathbb{N} \rightarrow \bigcup_{f \in \mathcal{V}} \text{Dom}(f)$ maps a fluent $f \in \mathcal{V}$ and a time point $t \in \mathbb{N}$ to the value of $f$ at time $t$ under $\rho$. It is defined as:
\begin{equation*}
\xi_\rho(f, t) = 
\begin{cases}
\textit{init}(f) & \text{if } t = 0\\
v & \text{if } \exists a \in \mathcal{A} \text{ s.t. } p(a) \wedge \\&  (\kappa, f:=v) \in \textit{eff}(a) \wedge \rho(\kappa) = t\\ 
\xi_\rho(f, t-1) & \text{otherwise}
\end{cases}   
\end{equation*}
\end{definition}

\noindent
Intuitively, $\xi_\rho(f, t)$ gives the value of fluent \textit{f} at time \textit{t}, set by the most recent activity \textit{a} ending by \textit{t} that updates \textit{f}. If none exists, it returns the initial value of \textit{f}.

\noindent
Validity and optimality of $\rho$ can then be defined as follows.

\begin{definition}\label{def:schedule-validity}
Let the set of activities active at time $t$ under schedule $\rho$ be $\mathcal{A}_\rho^t = \{a \in \mathcal{A} \mid p(a) \wedge s(a) \le t \le e(a)\}$.
A schedule $\rho$ is \textbf{valid} for an OS $\phi$ if it is non-conflicting and the following conditions hold.
\begin{enumerate}
    \item $\forall a \in \mathcal{A}$, $\neg p(a) \lor e(a) - s(a) \in [lb_a, ub_a]$, i.e., if the activity is present, its duration satisfies the duration bounds. 
    \item For each $r \in \mathcal{R}$, $\forall t \in \mathbb{N}. \sum_{a \in \mathcal{A}_\rho^t} \gamma_r^a \le \lambda_r$, i.e., total resource demand at any time does not exceed availability.
    \item Constraints in $\mathcal{C}_{\textit{t}}$ are satisfied using standard Boolean logic, with the value of atoms defined as follows:
    \begin{itemize}
        \item $a.present$ is true iff $p(a)$ (presence);
        \item $\kappa_1 - \kappa_2 \le \Delta t$ iff $\rho(\kappa_1) - \rho(\kappa_2) \le \Delta t$ (precedence).
    \end{itemize}
    \item Constraints in $\mathcal{C}_{\textit{f}}$ are satisfied: given $([\kappa_1, \kappa_2], a, f:= v) \in \mathcal{C}_{\textit{f}}$, either $\neg p(a)$ or $\forall t \in [\rho(\kappa_1), \rho(\kappa_2)]$, $\xi_\rho(f, t) = v$.
\end{enumerate}
\end{definition} 
\begin{definition}\label{def:schedule-optimality}
Given an OS $\phi$, a set of schedules $\mathcal{S}$, and a function $\textit{opt}: \mathcal{S} \rightarrow \mathbb{R}$ to be minimized, $\rho \in \mathcal{S}$ is \textbf{optimal} for $\phi$ if it is valid for $\phi$ and for every other valid schedule $\rho' \not= \rho \in \mathcal{S}$, $\textit{opt}(\rho) \le \textit{opt}(\rho')$. 
\end{definition}

We now incorporate motion activities, i.e., tasks involving object movement subject to motion constraints.

\begin{definition}\label{def:samp} A \textbf{Scheduling and Motion Planning} (SAMP) problem is a tuple $\psi = \langle\phi, \mathcal{O}, \mathcal{W}, \mathcal{Q}, u, i, \textit{mc}\rangle$, where:

\begin{itemize} 

    \item $\phi = \langle\mathcal{V}, \mathcal{A}, \mathcal{R}, \mathcal{C}, \textit{eff},  \textit{init}\rangle$ is an OS as per Definition~\ref{def:scheduling-with-fluents-and-effects}.

    \item $\mathcal{O} \subseteq \mathcal{R}$ is a set of movable objects, where each object $o \in \mathcal{O}$ is characterized by a geometric model $g_o$ and a control model $u_o$, with $\lambda_o = 1$ (only one is available). 
    
    \item $\mathcal{W} \subseteq \mathbb{R}^N$ (N = 2 or N = 3) is the workspace, i.e., the volume of reachable end-points for objects in $\mathcal{O}$. $\mathcal{W}_{free}$ is the portion of $\mathcal{W}$ that is free from fixed obstacles.

    \item $\mathcal{Q}$ is the configuration space, with $\mathcal{Q}_o \subseteq \mathcal{Q}$ the subset of $\mathcal{Q}$ representing the configurations that $o \in \mathcal{O}$ may assume given its motion model. 
    $occ(o, q) \subseteq \mathcal{W}_{free}$ is the set of points in $\mathcal{W}_{free}$ occupied by $o$ when in $q \in \mathcal{Q}_o$. 
    
    \item $i: \mathcal{O} \rightarrow \mathcal{Q}_o$ is a function that assigns to a movable object $o \in \mathcal{O}$ an initial configuration $q \in \mathcal{Q}_o$.

    \item $\textit{mc}: \mathcal{A}  \rightarrow \textit{mot}$ associates an activity $a$ to its motion constraint, where a motion constraint can be $\bot$ (indicating no motion constraint) or a tuple $\langle o_a, q^S_a, q^G_a\rangle$, where:
    \begin{itemize}
        \item $o_a \in \mathcal{O}$ is the movable object involved in the activity;
        \item $q^S_a, q^G_a \in \mathcal{Q}_{o_a}$ are the configurations it must assume at the start and the end of the activity, respectively.
    \end{itemize}
\end{itemize}
For each activity $a$ where $mc(a) \not= \bot$, we set $\gamma_{o_a}^a = 1$, i.e., each activity moving object $o$ uses the resource $o$.
\end{definition}

\noindent
This definition adds motion constraints to an OS problem. Since a trajectory $\tau(a): \mathbb{R}_{\ge 0} \rightarrow \mathcal{Q}$ specifies the configuration of object $o_a$ at each time $t \in [s(a), e(a)]$, describing its continuous motion from $q^S_a$ to $q^G_a$, we now extend solution schedules to handle SAMP problems.

\begin{definition}\label{def:samp-schedule}
A \textbf{SAMP schedule} for $\psi = \langle\phi, \mathcal{O}, \mathcal{W}, \mathcal{Q}, u, i, \textit{mc}\rangle$ is a tuple $\pi = \langle p, s, e, \tau \rangle$, with $\langle p, s, e \rangle = \rho$ a schedule for the OS problem $\phi$, and $\tau : \mathcal{A} \rightarrow \mathbb{R}_{\ge0} \rightarrow \mathcal{Q} \cup \{\bot\}$ a function that assigns to each $a \in \mathcal{A}$ a trajectory for the movable object $o_a$, if $\textit{mc}(a) \not= \bot$.
\end{definition}

\noindent
Note that we use (as customary) integer time for scheduling and real time for motion trajectories, following common practice in each domain; uniforming the time domain would require only minor adjustments in the formalization.  

Let $o_a$ be the object moved by activity $a\in\mathcal{A}$, i.e., whose motion constraint is $\textit{mc}(a) = \langle o_a$, $q^S_a, q^G_a\rangle$. A SAMP schedule is \textbf{non-conflicting} if $\rho$ is non-conflicting and there exists no time $t \in \mathbb{R}$ and activities $a_1 \not = a_2 \in \pi$ such that $o_{a_1} = o_{a_2}$ and $s(a_1) \le t \le e(a_1) \wedge s(a_2) \le t \le e(a_2)$, i.e., two activities moving the same object do not overlap in time.
Note that any valid OS schedule satisfies this condition, as movable objects are modeled as unary resources.

We now define a function that maps object configurations over time under a non-conflicting SAMP schedule.

\begin{definition}
Let $\pi$ be a non-conflicting SAMP schedule, and let the sequence of motions moving $o$ be $\mathcal{A}_\pi^o = \langle a \in \mathcal{A} \mid p(a) \wedge o = o_a \rangle = \langle a_0, a_1, \dots, a_n \rangle$,
ordered so that for any $a_i, a_j \in \mathcal{A}_\pi^o$, if $a_i$ precedes $a_j$ in $\pi$, then $s(a_i) < s(a_j)$. 
The \textbf{evaluation function} $\zeta_\pi : \mathcal{O} \times \mathbb{R}_{\ge 0} \rightarrow \mathcal{Q}$, which returns the configuration of $o \in \mathcal{O}$ at time $t \in \mathbb{R}_{\ge 0}$, is defined as:
\[
\scriptsize
\zeta_\pi(o, t) =
\begin{cases}
    i(o) & \text{if } t < s(a_0),\\
    \tau(a_i)(t) & \text{if } s(a_i) \le t \le e(a_i), i \in \{0, \text{...}, n\},\\
    \tau(a_i)(e(a_i)) & \text{if } e(a_i) < t < s(a_{\text{i+1}}), i \in \{0, \text{...}, \text{n-1}\},\\
    \tau(a_{n})(e(a_{n})) & \text{if } t > e(a_{n}).
\end{cases}
\]
\end{definition}

The validity of $\pi$ is then defined as follows.

\begin{definition}\label{def:samp-schedule-validity}
A non-conflicting SAMP schedule $\pi$ is \textbf{valid} for the SAMP problem $\psi$ if $\rho$ is valid for the OS problem $\phi$ as per Definition~\ref{def:schedule-validity} and the following constraints hold:
\begin{enumerate}
    \item $\forall t \in \mathbb{R}_{\ge 0}$, $\forall o_i,o_j \in \mathcal{O}$ such that $o_i \neq o_j$, $occ(o_i, \zeta_\pi(o_i, t)) \cap occ(o_j, \zeta_\pi(o_j, t)) = \emptyset$, i.e., object motions are collision-free.
    \item $\forall o \in \mathcal{O}$, $\forall t \in \mathbb{R}_{\ge 0}$, the configuration $\zeta_\pi(o, t)$ lies on a trajectory that is dynamically feasible under the control model $u_o$ of $o$; i.e., that can be executed by its controller.
    \item $\forall o \in \mathcal{O}$, let $\mathcal{A}^o_\pi = \langle a_0, \dots, a_n \rangle$ be the sequence of activities in $\pi$ moving $o$. Then, $\tau(a_i)(e(a_i)) = \tau(a_{i+1})(s(a_{i+1}))$ for all $i \in {0, \dots, n-1}$, ensuring that the trajectory of $o$ is continuous in space-time among all activities moving it.
\end{enumerate}
\end{definition}

One interesting case is the one with no fluents nor effects. This is practically relevant because not all schedulers support fluents (e.g., OR-Tools~\cite{cpsatlp}).
\begin{definition}
An \textbf{OS problem without fluents} is defined as OS $\phi$ with $\mathcal{V} = \emptyset$.
Accordingly, a \textbf{SAMP problem without fluents} is defined as SAMP $\psi$ with $\mathcal{V} = \emptyset$.
\end{definition} 

In this paper, we propose a framework for SAMP that supports different schedulers by expressing constraints either using fluent conditions and effects or just using precedence constraints. The framework is detailed in the next section.

\section{The Core Framework}
\begin{tikzpicture}[overlay, remember picture]
  \draw[overlay, fill opacity=0.2, fill=gray, draw=none]
    ($(pic cs:a) + (-0.05,0.3)$) 
    rectangle 
    ($(pic cs:b) + (7.35,0.25)$);
\end{tikzpicture}

\begin{algorithm}[t]
\DontPrintSemicolon
\small
\caption{The Core Framework}\label{alg1:meta-engine}\label{alg3:GetMotionOrRefine}
\Begin($\textsc{Solve}{(}\psi, \textit{opt}, t_p, \textit{timeout}{)}$)
{
    $\psi '\gets \psi$ \quad $it \gets 0$ 
    
    \While{Now() $<$ timeout}{
        $\rho$, \textit{status} $\gets$ \textit{get-schedule}($\psi'$, \textit{opt}) \Comment*{Invoke the scheduler} \label{alg1:get-schedule}

        \uIf{\textit{status} $\in$ [\texttt{VALID}, \texttt{OPTIMAL}]}{
            
        $\tau(\rho) \gets \emptyset$
    
        \textit{conf}(o) $\gets i(o) \quad \forall o \in \mathcal{O}$ \label{alg1:set-config}
        
         \ForEach(\Comment*[f]{Check each parallel motion group}){$\mathcal{G} \in \mathcal{P}(\rho)$} 
        {
            \tikzmark{a}
            \ForEach(\Comment*[f]{Geometric check of each activity}\label{alg1:check-geom}){$a \in \mathcal{G}$}
            { 
        
                    \textbf{if } $\neg$  \textsc{GetMotionOrRefine}($\{a\}$, $\psi'$, \textit{conf}, \texttt{GEOM}, $t_p$) \textbf{ then goto} 20
            }
        
            \ForEach(\Comment*[f]{Temporal check of each activity}\label{alg1:check-temp}){$a \in \mathcal{G}$}
            {
        
                \textbf{if } $\neg$ \textsc{GetMotionOrRefine}($\{a\}$, $\psi'$, \textit{conf}, \texttt{TIME}, $t_p$) \textbf{ then goto} 20\label{alg1:last-optimization}

            }
            \tikzmark{b}
            $\tau(\mathcal{G})\gets$ \textsc{GetMotionOrRefine}($\mathcal{G}$, $\psi'$, \textit{conf}, \texttt{ALL}, $t_p$) \label{alg1:get-motion-or-refine}
    
            \textbf{if } $\tau(\mathcal{G}) \not= \emptyset$ \textbf{ then } $\tau(\rho) \gets \tau(\rho) \cup \tau(\mathcal{G})$  \textbf{ else goto} 20 
    
            \textit{update(conf, $\mathcal{G}$)} \label{alg1:config-update}
        }
            \textbf{if } $\tau(\rho) \neq \emptyset$ \textbf{ then return } $\langle \rho, \tau\rangle$ \Comment*{Return  SAMP schedule} \label{alg1:plan}
            
        }
        \Else{
            \textbf{if } $it == 0$ \textbf{ then return } $\texttt{UNSOLVABLE}$ \label{alg1:unsolvable}

            $t_p \gets 2 \cdot t_p$ \quad $\psi ' \gets \psi$ \Comment*{Double timeout and reset $\psi$} \label{alg1:double-time}\label{alg1:reset-constraints}

        }
        $it \gets it + 1$

    }
    \Return $\texttt{INCOMPLETE}$
}

\vskip 3pt

\Begin($\textsc{GetMotionOrRefine}{(} \mathcal{G}, \psi, \textit{conf}, \textit{refs}, t_p{)}$)
{
    $\tau(\mathcal{G}) \gets \emptyset$ \quad \textit{path-found} $\gets \texttt{True}$
    
    $s_{\textit{min}} \gets \min(\{s(a) | a \in \mathcal{G}\})$ \label{alg3:s-min}

    $\mathcal{C}_\mathcal{G} \gets \{ (mc(a), \delta_a = s(a) - s_{min}) | a \in \mathcal{G}\}$ \label{alg3:group-constraints}
    
    \uIf{refs = \texttt{GEOM} $\wedge \mathcal{G} = \{a\}$}{
        \textit{path-found}, $\Sigma$, $\Omega$ $\gets$ \textit{get-path}($\mathcal{C}_\mathcal{G}$, \textit{conf}, $t_p$) \label{alg2:get-path}
    }
    \Else{
        $\tau(\overline{\mathcal{G}})$,  $\overline{d}$, $\overline{\delta}$, $\Sigma$, $\Omega$ $\gets$ \textit{get-motion}($\mathcal{C}_\mathcal{G}$, \textit{conf}, $t_p$) \Comment*{$\overline{\mathcal{G}} \subseteq \mathcal{G}$}\label{alg2:get-motion}
    }

    \uIf{$\tau(\overline{\mathcal{G}}) = \emptyset \vee \neg\textit{path-found}$}
    {
        $\psi$.\textit{add-geometric-refinements}($ \overline{\mathcal{G}}$, $\Sigma$, $\Omega$, \textit{conf}) \label{alg2:geometric-ref}
    }
    \uElseIf{refs = \texttt{TIME} or refs = \texttt{ALL}}
    {
        \uIf(\label{alg3:check-timings}){$\neg\bigwedge_{a\in \overline{\mathcal{G}}} \overline{d}_a + \overline{\delta}_a \leq d_a + \delta_a$} 
        {
        
            $\overline{\mathcal{C}}_{\overline{\mathcal{G}}} \gets \{ (mc(a), \overline{\delta}_a) | a \in \overline{\mathcal{G}}\}$ \label{alg2:group-constraints}
            
            $\psi$.\textit{add-temporal-refinements}($\overline{\mathcal{G}}$, $\overline{\mathcal{C}}_\mathcal{G}$, $\overline{d}$, \textit{conf})\label{alg2:temporal-ref}

            \Return $\emptyset$
        }
    }
    
    \Return $\tau(\mathcal{G})$ \label{alg3:return-refs}
}
\end{algorithm}

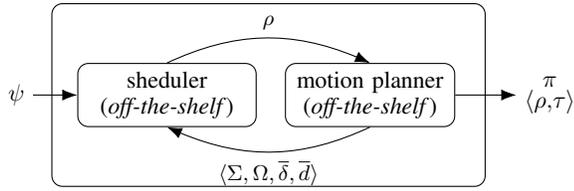
\begin{figure}[t]
\centering
\begin{tikzpicture}[scale=0.9, transform shape]
    \node[draw, rounded corners, minimum width=6.4cm, minimum height=2.7cm] (outer) {};

    \node[draw, rounded corners, rectangle,
          minimum width=2.5cm, minimum height=0.8cm,
          text width=2.4cm, align=center]
          (scheduler) at ([xshift=-1.5cm]outer.center)
          {sheduler \\(\textit{off-the-shelf})};

    \node[draw, rounded corners, rectangle,
          minimum width=2.5cm, minimum height=0.8cm,
          text width=2.2cm, align=center]
          (motionplanning) at ([xshift=1.5cm]outer.center)
          {motion planner \\ (\textit{off-the-shelf})};

    \node[align=center] (problem) at ([xshift=-0.9cm]scheduler.west) {$\psi$};

    \node[align=center] (samp-plan) at ([xshift=1.4cm]motionplanning.east) 
        {\LARGE $\substack{\pi \\ \langle \rho, \tau \rangle}$};

    \draw[-{Latex[length=2mm]}] (problem) -- (scheduler);
    \draw[-{Latex[length=2mm]}] (motionplanning) -- (samp-plan);

    \draw[-{Latex[length=2mm]}, bend left=25] 
        (scheduler.north) to node[above] {\small $\rho$} (motionplanning.north);

    \draw[-{Latex[length=2mm]}, bend left=25] 
        (motionplanning.south)
        to node[below] {\small $\langle \Sigma, \Omega, \overline{\delta}, \overline{d} \rangle$} 
        (scheduler.south);

\end{tikzpicture}
\caption{Our framework. Given a SAMP problem $\psi$, the scheduler sends a candidate schedule $\rho$ to the motion planner. If invalid, the planner returns \textit{geometric} (unreachable configurations $\Sigma$ and obstacles $\Omega$)  and \textit{temporal} (new delays $\overline{d}$ and durations $\overline{\delta}$) refinements until a valid SAMP schedule $\pi$ is found (with trajectories $\tau$), if one exists.}
\label{fig:core-framework}
\end{figure}

Our SAMP framework incrementally learns symbolic abstractions of motion tasks (Algorithm~\ref{alg1:meta-engine}). It interleaves an off-the-shelf scheduler, which proposes a motion-agnostic candidate schedule $\rho$ (line~\ref{alg1:get-schedule}), with an (instrumented) off-the-shelf motion planner that checks the feasibility of $\rho$ via \textsc{GetMotionOrRefine} (line~\ref{alg1:get-motion-or-refine}). The motion planner returns valid trajectories, used to decorate the motion activities in  $\rho$, if they exist; otherwise, it provides spatio-temporal refinements for the next scheduling iteration (Figure~\ref{fig:core-framework}).

The initial problem submitted to the scheduler is the OS $\phi$ from the SAMP problem enriched with simple constraints to ensure object trajectories are continuous (third condition of Definition~\ref{def:samp-schedule-validity}). This can be achieved either by a fluent tracking each object’s configuration and imposing a condition for each motion activity, or via precedence constraints restricting the admissible ordering of motion activities.
For the motion planning problem, solving it monolithically would be computationally infeasible; thus, we divide the schedule into \textit{parallel motion groups}: subsets of activities that can interfere with each other but are independent from other groups. 
\begin{definition}
Two activities $a, b \in \mathcal{A}$ are \textbf{parallel} in $\rho$ if $p(a)$ and $p(b)$ hold, and $\exists t \in \mathbb{R}$ such that $s(a) \le t \le e(a) \wedge s(b) \le t \le e(b)$. They are further defined as  \textbf{motion-parallel} if they are parallel, $\textit{mc}(a) \neq \bot$, and $\textit{mc}(b) \neq \bot$.
\end{definition}

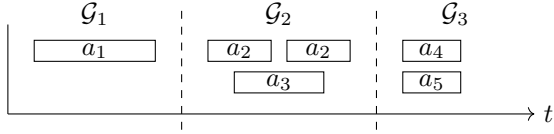
\begin{figure}[t]
    \centering
    \begin{tikzpicture}[scale=0.7]

        \draw[->] (0,0) -- (10,0) node[right] {$t$};
        \draw[-] (0,0) -- (0,1.7);
        
        \draw[dashed] (3.3,-0.3) -- (3.3,2.1);
        \draw[dashed] (7.0,-0.3) -- (7.0,2.1);
        
        \node at (1.65,1.9) {$\mathcal{G}_1$};
        \node at (5.15,1.9) {$\mathcal{G}_2$};
        \node at (8.5,1.9) {$\mathcal{G}_3$};
        
        \draw (0.5,1.4) rectangle (2.8,1) node[pos=.5] {$a_1$};
        
        \draw (3.8,1.4) rectangle (5.0,1) node[pos=.5] {$a_2$};
        \draw (5.3,1.4) rectangle (6.5,1) node[pos=.5] {$a_2$};
        \draw (4.3,0.8) rectangle (6.0,0.4) node[pos=.5] {$a_3$};
        
        \draw (7.5,1.4) rectangle (8.6,1) node[pos=.5] {$a_4$};
        \draw (7.5,0.8) rectangle (8.6,0.4) node[pos=.5] {$a_5$};
            
    \end{tikzpicture}
    \caption{Parallel motion groups $\mathcal{P}(\pi) = \{\mathcal{G}_1, \mathcal{G}_2, \mathcal{G}_3\}$.}
    \label{fig:parallel-motion-groups}
\end{figure}

A \textbf{\textit{parallel motion group}} $\mathcal{G}(\rho)$ ($\mathcal{G}$ for the rest of paper) is a maximal set of motion-parallel activities in $\rho$, with $\mathcal{P}(\rho)$ the set of all such groups (see Figure~\ref{fig:parallel-motion-groups}).
Given $\mathcal{G} \in \mathcal{P}(\rho)$, and the configuration of each object at the start time of $\mathcal{G}$ (\textit{conf}, line~\ref{alg1:set-config}), \textsc{GetMotionOrRefine} either: 

\begin{itemize}
    \item \textbf{Gets $\tau(\mathcal{G})$}: finds collision-free, temporally consistent trajectories for all the activities in $\mathcal{G}$. 

    \item \textbf{Adds geometric refinements}: exploits the motion planner's exploration to identify unreachable locations and blocking obstacles for those activities that are spatially infeasible and adds them as new constraints of $\psi$.
    
    \item \textbf{Adds temporal refinements}: adds constraints on execution durations and inter-activity delays for motions that are geometrically feasible but violate the temporal constraints imposed by the scheduler, ensuring safe synchronization.
    
\end{itemize}
If $\tau(\mathcal{G})$ exists, the planner updates each object's configuration to the goal of the last activity moving it (line~\ref{alg1:config-update}) and goes to the next group.
If all groups are valid, a SAMP plan $\pi = \langle \rho, \tau\rangle$ is returned (line~\ref{alg1:plan}), optimal under Definition~\ref{def:schedule-optimality} (assuming optimality depends solely on $\rho$, not on trajectory optimization); otherwise, new constraints are generated. 

Before describing how constraints are computed, we give a few final details of the core framework. Since many motion planners are sample-based, they may fail to terminate if no path exists or even time out despite a solution existing (\textit{false-negative} sensitivity). This means a schedule marked as unsolvable may be truly infeasible (line~\ref{alg1:unsolvable}) or \textit{falsely} unsolvable due to insufficient sampling. To address this, we impose a timeout $t_p$ per planner call and, when a solution is initially unsolvable, double $t_p$ up to an upper limit, reset the refinements (line~\ref{alg1:reset-constraints}), and restart, mitigating false negatives with minimal added complexity.
To further improve efficiency, we cache the trajectories of activities or groups whose motion constraints have already been validated. Before re-evaluating any constraints, the cache is checked to avoid redundant computations. This cache persists across restarts, enabling the reuse of previously validated trajectories and reducing computational overhead.
The final optimization (gray box of Algorithm~\ref{alg1:meta-engine}) reduces the cost of evaluating entire groups via a layering architecture that, before checking entire groups (Layer 2), validates their single activities (Layer 1). Each activity undergoes a geometric feasibility check  (line~\ref{alg1:check-geom}), followed by a temporal feasibility check (line~\ref{alg1:check-temp}). Both operate as specialized instances of \textsc{GetMotionOrRefine} (see Algorithm~\ref{alg3:GetMotionOrRefine}), but applied to single activities and their respective refinements. In this context, the geometric check effectively verifies the existence of a path, simplifying the motion planner’s role to that of a path finder (line~\ref{alg2:get-path}).
This pre-check boosts performance by avoiding unnecessary and expensive group-level synchronization checks, as shown in our experimental evaluation.

\section{Geometric and Temporal Refinements}
In this section, we detail how \textsc{GetMotionOrRefine} computes and formulates the spatio-temporal refinements.

\myparagraph{Geometric Refinements}
Let $\mathcal{Q}_\psi \subseteq \mathcal{Q}$ be the finite subset of configurations relevant to the problem, i.e., those actually involved in the motion activities of the problem, defined as:
\[
\small
\mathcal{Q}_\psi = \{q^S_a, q^G_a | a \in \mathcal{A}, mc(a) = \langle o_a,  q^S_a, q^G_a\rangle\}, 
\]

If the motion planner fails to find a trajectory for the group $\mathcal{G}$ (line~\ref{alg2:get-motion}), or a path in the case of a singleton group (line~\ref{alg2:get-path}), it reports the spatial conflicts encountered, being:

\begin{itemize}
    \item $\Sigma = \{\sigma_a \subseteq \mathcal{Q}_\psi \mid a \in \mathcal{G}\}$: the set of \textbf{unreachable configurations} for each activity $a \in \mathcal{G}$. For an activity $a$ moving $o_a$ from $q^S_a$ to $q^G_a$, the reachable set is:
    \[
        \widetilde{\sigma}_a = \left\{ q \in \mathcal{Q}_\psi \mid \textit{occ}(o_a, q) \subseteq \textit{reach}(q^S_a) \right\},
    \]
    where $\textit{reach}(q^S_a)$ is the region reachable from $q^S_a$, computed from the area explored by the motion planner. $\sigma_a = \mathcal{Q}_\psi \setminus \widetilde{\sigma}_a$ includes all unreachable configurations, i.e., all configurations outside $\textit{reach}(q^S_a)$,  including $q^G_a$ if not reachable.
    
    \item $\Omega = \{\omega_a \subset \mathpzc{O} \mid a \in \mathcal{G}\}$: the set of \textbf{blocking obstacles} identified by collision checking for each $a \in \mathcal{G}$. Given $a$, when using a sampling-based motion planner, an object $o \in \mathpzc{O}$ is added in $\omega_a$ if $o \neq o_a$ and a collision was detected between $o$ and $o_a$ when extending the motion-tree of $o_a$, i.e., $o$ blocks the expansion of possible motions of $o_a$.
\end{itemize}
Such spatial conflicts are used to inform the scheduler that the configuration of at least one blocking obstacle must be modified to make an otherwise unreachable configuration reachable. This refinement is performed by $\psi$.\textit{add-geometric-refinements}($\mathcal{G}, \Sigma, \Omega, \textit{conf}$) (line~\ref{alg2:geometric-ref}).
To formalize this constraint, we first define the set of \textbf{rivals} of a parallel motion group $\mathcal{G}$ as the motion activities not in $\mathcal{G}$:
\[
 \small
\rivals = \{ r \in \mathcal{A} \;|\; r \notin \mathcal{G}, mc(r) \not = \bot \}.
\]
An activity $r \in \rivals $ does not overlap with $\mathcal{G}$ ($\neg\textit{overlaps}(r, \mathcal{G})$) if it does not exist or is entirely scheduled before or after $\mathcal{G}$:
\[
\small
r.\text{present} \rightarrow [\bigwedge_{a \in \mathcal{G}}(r.\text{end} < a.\text{start}) \vee \bigwedge_{a \in \mathcal{G}}(r.\text{start} > a.\text{end})]
\]
Given $\mathcal{G}$, let $\mathcal{G}_o = \{a \in \mathcal{G} \mid o_a = o\}$ be the activities in $\mathcal{G}$ moving $o$, and let $a^o_{min}$ be the first activity moving $o$ (i.e., $s(a^o_{min}) \le s(a)$ $ \forall a \in \mathcal{G}_o$). We define the refinement condition formula $\textsc{RCond}(\mathcal{G})$ as:
$$
 \!\bigwedge_{a \in \mathcal{G}} \! a.\text{present} \wedge \!\!\!\!\!\!\!\!\!\! \bigwedge_{\substack{o \in \mathcal{O}\\a \in \mathcal{G}_o \setminus \{a^o_{min}\}}} \!\!\!\!\!\!\!\!\!\!\! a^o_{min}.\text{end} \le a.\text{start} \wedge \!\! \bigwedge_{r \in \rivals} \!\! \neg \textit{overlaps}(r, \mathcal{G})
$$
It specifies the condition under which all activities in $\mathcal{G}$ are scheduled, with each activity moving an object, and all rival activities not overlapping $\mathcal{G}$ (temporal and geometric refinements thus depend only on the activities of the group and the initial configuration of each movable object).
\noindent
The new constraint sent to the scheduler is then defined as follows: 
$$
\small
\textsc{RCond}(\mathcal{G}) \! \rightarrow \!\!\!\!\!\! \bigvee_{b \in \mathcal{G}, o \in \omega_b} \!\!\!\!\!\! \textsc{ChConf}(b, \textit{conf}(o))
$$
with $o \in \omega_b$ being a blocking obstacle for $b \in \mathcal{G}$ and \textit{config(o)} its current configuration.  That is, if $\mathcal{G}$ is scheduled and none of its rival activities overlap with it, then the configuration of at least one blocking obstacle must change.
This constraint can be encoded either using fluents or without them.

\begin{figure}[!t]
\centering
\begin{tikzpicture}
\node (mainfig) {
\begin{tikzpicture}[scale=0.6]
    \draw[->] (0,-0.3) -- (9.5,-0.3);
    \draw[-] (0,-0.3) -- (0,3.9);

    \node at (9.2,-0.6) {\footnotesize $t$};
    
    \node at (0,-0.6) {\footnotesize $s_\text{min}$};

    \draw[thick] (0,0.2) rectangle (5,0.7);
    \node at (2.5,0.45) {\footnotesize $a$};

    \draw[thick] (6.2,0.2) rectangle (8.6,0.7);
    \node at (7.4,0.45) {\footnotesize $c$};

    \draw[thick] (2,1.2) rectangle (7,1.7);
    \node at (4.5,1.45) {\footnotesize $b$};

    \draw[<->] (2,2) -- node[above,yshift=-2pt] {\footnotesize $d_b$} (7,2);

    \draw[<->] (0,2.85) -- node[above,yshift=-2pt] {\footnotesize $\overline{\delta}_b$} (3.5,2.85);
    \draw[<->] (0,1.45) -- node[above,yshift=-2pt] {\footnotesize $\delta_b$} (2,1.45);

    \draw[thick, fill=gray!20] (3.5,2.6) rectangle (9,3.1);
    \node at (6,2.85) {\footnotesize $b$};

    \draw[<->] (3.5,3.4) -- node[above,yshift=-2pt] {\footnotesize $\overline{d}_b$} (9,3.4);

    \draw[dashed] (5,-0.3) -- (5,1.0);
    \draw[dashed] (6.2,-0.3) -- (6.2,1.0);
    \draw[dashed] (8.6,-0.3) -- (8.6,1.0);
    \draw[dashed] (2,-0.3) -- (2,2.3);
    \draw[dashed] (7,-0.3) -- (7,2.3);
    \draw[dashed] (3.5,-0.3) -- (3.5,3.9);
    \draw[dashed] (9,-0.3) -- (9,3.9);
\end{tikzpicture}
}; 

\node[anchor=north west, xshift=-23mm, yshift=-5mm]
    at (mainfig.north west)
    {\includegraphics[width=0.28\columnwidth]{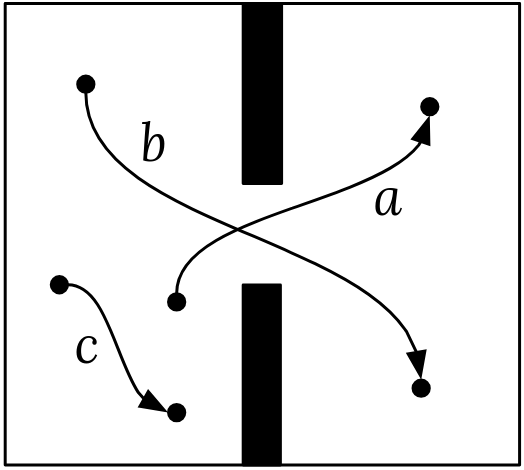}};

\end{tikzpicture}

\caption{
Temporal refinement for $\mathcal{G}=\{a,b,c\}$, starting at $s_{\min}$ (start of $a$). 
The motion planner delays the start of $b$ (from $\delta_b$ to $\overline{\delta}_b$) and increases its duration (from $d_b$ to $\overline{d}_b$).
}
\label{fig:temporal-info}
\end{figure}
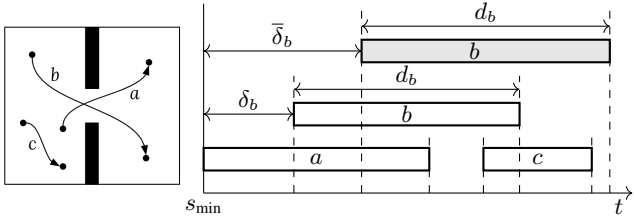

In the case of fluents, for each activity $b \in \mathcal{G}$ we introduce into $\psi$ a corresponding auxiliary activity $b'$. The activity $b'$ has the same start and end times as $b$, and it includes a precondition on the fluent $f_o$, which represents the configuration of the object $o \in \omega_b$. This precondition requires that the value of $f_o$ be \emph{different} from the blocking configuration $\mathit{conf}(o)$. Intuitively, the role of $b'$ is to ensure that, whenever $b$ is relevant for execution, the object $o$ is not in the configuration that would block or invalidate the execution of $b$. Then, our refinement requires $b'$ to be present:
\begin{equation*}
\small
\begin{aligned}
\textsc{ChConf}(b, \textit{conf}(o)) =~ & b'.\text{present} \wedge (b'.\text{start} = b.\text{start}) \\
&\wedge (b'.\text{end} = b.\text{end})
\end{aligned}
\end{equation*}

\begin{figure*}
    \centering
    \includegraphics[width=\linewidth]{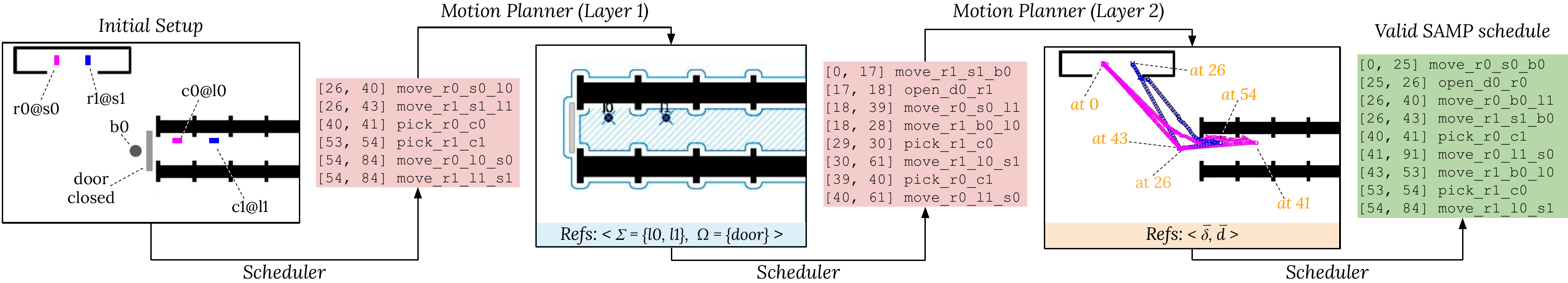}
    \caption{A logistics scenario with two robots ($r_0, r_1$) delivering two items ($c_0, c_1$) from $l_0$ and $l_1$. The first schedule is infeasible as $\Sigma=\{l_0,l_1\}$ is blocked by $\Omega=\{\text{door}\}$ (Layer~1, RRT). The second schedule is geometrically feasible but trajectories need updated delays $\overline{\delta}$ and durations $\overline{d}$ (Layer~2, ST-RRT*). Such motion planning's feedback leads to a final valid SAMP schedule.}
    \label{fig:samp_example}
\end{figure*}

Without fluents, \textit{helper activities} $\mathcal{H} = \{h \in \mathcal{A} \mid \textit{mc}(h) = \langle o_h, q^S_h, q^G_h\rangle \wedge q^G_h \neq \textit{conf(o)}\}$ move $o$ to any state different from \textit{conf(o)} and \textit{deleter activities} $\mathcal{X} = \{x \in \mathcal{A}  \mid \textit{mc}(x) = \langle o, q^S_x, q^G_x\rangle \wedge q^G_x = \textit{conf(o)}\}$ place $o$ in the blocking configuration \textit{conf(o)}. We define $\textsc{ChConf}(a, \textit{conf}(o))$ as:
\begin{equation*}
\small
\begin{split}
\textsc{Del}(b, \textit{conf}(o))\vee \bigvee_{h \in \mathcal{H}} \left( \phantom{\bigvee_{h \in \mathcal{H}}\!\!\!\!\!\!\!\!\!\!\!}(h.\text{present} \wedge (h.\text{end} < b.\text{start})) \wedge \right.\\[-6pt]
\left. \bigwedge_{x \in \mathcal{X}} (x.\text{present} \rightarrow (x.\text{end} < h.\text{start}) \vee (x.\text{start} > b.\text{end})) \right) 
\end{split}
\end{equation*}
with $\textsc{Del}(b, \textit{conf}(o))$ being 
\begin{equation*}
\small
\begin{cases}
\bigwedge_{x \in \mathcal{X}} x.\text{present} \rightarrow (x.\text{start} > b.\text{end}) & \text{if } \textit{conf}(o) \not= i(o)\\
False & \text{otherwise}
\end{cases}   
\end{equation*}
Intuitively, this constraint requires that either there is no deleter activity before $b$ and the initial configuration of $o$ is different from \textit{conf(o)}, or that there exists an helper activity occurring before $b$ and any deleter activities happen before the helper or after $b$. In essence, obstacles must be removed before executing a motion they would otherwise block.

To improve performance and avoid repeated computation, we propagate and cache reachability information for all equivalent objects, i.e., sharing the same geometry and control, located within the same reachability area. For singleton groups $\mathcal{G} = \{a\}$, we generalize geometric constraints to all activities moving equivalent objects from the same reachability area $\widetilde{\sigma}_a$ toward a target configuration within the set of configurations $\sigma_a$ deemed unreachable by $a$.

\myparagraph{Temporal Refinements}
Geometric feasibility does not guarantee temporal feasibility: scheduled times may differ from those needed for execution, and parallel motions may require delay adjustments for collision-free synchronization.

\textsc{GetMotionOrRefine} performs this check. 
It computes the earliest start time $s_{\textit{min}} = \min\{s(a) \mid a \in \mathcal{G}\}$ among the activities of the group (line~\ref{alg3:s-min}), and it collects $\mathcal{C}_{\mathcal{G}}$ (line~\ref{alg3:group-constraints}), i.e., the set of motion constraints with scheduled start times $\delta_a = s(a) - s_{\textit{min}}$.
If for at least one subset $\overline{\mathcal{G}} \subseteq \mathcal{G}$ (possibly $\mathcal{G}$ itself) the motion planner identifies trajectories $\tau(\overline{\mathcal{G}})$ that are geometrically feasible but fail to satisfy the timing constraints imposed by the scheduler (line~\ref{alg2:get-motion}), then \textit{get-motion} immediately returns (as this is sufficient to prove the whole candidate schedule is unfeasible) and outputs:

\begin{itemize}
    \item $\overline{d} = \{\overline{d}_a \in \mathbb{R}_{\geq 0} \mid a \in \overline{\mathcal{G}}\}$: new estimated durations.
    \item $\overline{\delta} = \{\overline{\delta}_a \in \mathbb{R}_{\geq 0} \mid a \in \overline{\mathcal{G}}\}$: new estimated delays from $s_{\textit{min}}$
\end{itemize}
These values are computed from the space-time trajectories generated by the motion planner. Specifically, given an activity $a$ moving $o_a$, and given the space-time sequence of states along its planned trajectory, we compute the actual motion start time $\overline{\delta}_a$ as the earliest timestamp at which $o_a$ exhibits non-negligible translational or angular displacement with respect to $s_{\textit{min}}$. We then determine the motion duration $\overline{d}_a$ by summing the time intervals $\Delta t$ from the first movement of $o_a$ until it comes to rest (see Figure~\ref{fig:temporal-info}).
In this case, \textsc{GetMotionOrRefine} verifies whether the needed timing $\overline{d}_a + \overline{\delta}_a$ does not exceed the scheduled $d_a + \delta_a$ for $a \in \overline{\mathcal{G}}$ (line~\ref{alg3:check-timings}). If this condition holds, the space-time trajectory is deemed valid (cached) and returned: we assume it is always possible for movable objects to pause. 
If the computed timings exceed the scheduled ones, $\psi$.\textit{add-temporal-refinments}($\overline{\mathcal{G}}$, $\overline{\mathcal{C}}_g$, $\overline{d}$, \textit{conf}) uses needed durations ($\overline{d}$), delays ($\overline{\delta}$, included in $\overline{\mathcal{C}}_g$), and current configurations (\textit{conf}) to add to the problem a new temporal refinement indicating that $\overline{\mathcal{G}}$ cannot be executed as scheduled unless at least one duration or delay is adjusted (line~\ref{alg2:temporal-ref}). Formally: 
\begin{align*}
\small
\textsc{RCond}(\mathcal{G}) \rightarrow \textsc{ChTime}(\overline{\mathcal{G}})
\end{align*}

\noindent
This means that if the parallel group $\mathcal{G}$ is scheduled and no rival $r \in \rivals$ overlaps with any $a \in \mathcal{G}$, the timing of activities in $\overline{\mathcal{G}}$ must be adjusted. Given $\overline{\delta}_a$, $\overline{d}_a$, and $\omega_a$, $\textsc{ChTime}(\overline{\mathcal{G}})$ is
\begin{equation*}
\small
\begin{split}
\bigvee_{\substack{a\in\mathcal{G}\\ o \in \mathcal{O}}} \!\! \textsc{ChConf}(a, \textit{conf}(o)) \vee
\bigvee_{a \in \overline{\mathcal{G}}} a.\text{start} - \min_{b\in \overline{\mathcal{G}}} (b.\text{start}) < \delta_a \vee \\
\bigvee_{a\in \overline{\mathcal{G}} | (d_a+\delta_a) < (\overline{\delta}_a + \overline{d}_a)} (a.\text{end} - \min_{b\in \overline{\mathcal{G}}} (b.\text{start})) \geq \overline{\delta}_a + \overline{d}_a
\end{split}
\end{equation*}
Thus, the scheduler must either require an object’s configuration to change before at least one group activity starts, advance the start of at least one activity in the subgroup, or extend the duration of some activity 
$a$ in the subgroup to at least the value $\overline{\delta}_a + \overline{d}_a$ estimated by the motion planner. 

As an example, consider the parallel motion group $\mathcal{G} = \{a, b, c\}$ of Figure~\ref{fig:temporal-info}, which starts at $s_{\min} = s(a)$ (the start time of $a$). To ensure $b$ is feasible when executed in parallel with $a$ ($\overline{\mathcal{G}} = \{a, b\}$), the motion planner schedules $b$ to start at $\overline{\delta}_b$ with an updated duration $\overline{d}_b$ (no obstacle obstructs $o_b$, the object moved by $b$). In this case, the motion planner must inform the scheduler that either (i) $b$ must be anticipated with respect to the current schedule (an option the scheduler has not yet requested the motion planner to evaluate), or (ii) $b$ must be assigned a new end time equal to $\overline{\delta}_b + \overline{d}_b$. Formally:
\begin{equation*}
\footnotesize
\textsc{ChTime}(\overline{\mathcal{G}}) =  b.\text{start} - s(a) < \delta_b \vee  b.\text{end} - s(a) \ge \overline{\delta}_b + \overline{d}_b .
\end{equation*}

For singleton groups, temporal refinements apply to all equivalent activities, as with geometric refinements. 

As a result, our framework synchronizes parallel motion activities by postponing starts, adjusting trajectories and durations, or executing stop-and-go maneuvers on the objects.

\paragraph{Formal guarantees.} Our framework is \textit{sound}: it returns a solution only if the schedule satisfies all constraints and the motion planner finds valid trajectories for all its motion activities. It is \textit{relatively optimal} for makespan optimization: if the motion planner computes duration-minimal solutions and the scheduler generates makespan-minimal plans, the resulting SAMP solution is makespan optimal. 
Assuming the motion planner is complete, \textit{relative completeness} (returning a solution if one exists) relies on showing that the learned constraints always prune the candidate schedule from the solution space of the scheduler and do not cut any valid solution. The first property follows from the presence of spatio-temporal refinements, while the second relies on the refinements being triggered by $\textsc{RCond}$
\ifextendedversion
 (see Appendix).
\else
.
\fi

\section{Experimental Evaluation}
\begin{table*}[t]
\renewcommand{\arraystretch}{1.2} 
\setlength\tabcolsep{2pt}
\resizebox{\textwidth}{!}{%
\begin{tabular}{l||lll|lll|lll|lll|lll|lll}
\textbf{Benchmark} &
  \multicolumn{3}{c|}{\textbf{CPSE} (no fluents) } &
  \multicolumn{3}{c|}{\textbf{Aries} (no fluents)} &
  \multicolumn{3}{c||}{\textbf{Aries} (with fluents) }&
  \multicolumn{3}{c|}{\textbf{CPSE-opt} (no fluents) }&
  \multicolumn{3}{c|}{\textbf{Aries-opt} (no fluents) }  &
  \multicolumn{3}{c}{\textbf{Aries-opt} (with fluents)} \\
&
  \multicolumn{1}{c|}{\#{\textit{sol}}} &
  \multicolumn{1}{c|}{$t$ [s] (\% $t_p$)} &
  \multicolumn{1}{c|}{\textit{refs}} &
  
  \multicolumn{1}{c|}{\#{\textit{sol}}} &
  \multicolumn{1}{c|}{$t$ [s]  (\% $t_p$)} &
  \multicolumn{1}{c|}{\textit{refs}} &
  
  \multicolumn{1}{c|}{\#{\textit{sol}}} &
  \multicolumn{1}{c|}{$t$ [s]  (\% $t_p$)} &
  \multicolumn{1}{c||}{\textit{refs}} &
  
  \multicolumn{1}{c|}{\#{\textit{sol}}} &
  \multicolumn{1}{c|}{$t$ [s]  (\% $t_p$)} &
  \multicolumn{1}{c|}{\textit{refs}} &
  
  \multicolumn{1}{c|}{\#{\textit{sol}}} &
  \multicolumn{1}{c|}{$t$ [s]  (\% $t_p$)} &
  \multicolumn{1}{c|}{\textit{refs}} &
  
  \multicolumn{1}{c|}{\#{\textit{sol}}} &
  \multicolumn{1}{c|}{$t$ [s]  (\% $t_p$)} &
  \multicolumn{1}{c}{\textit{refs}}  \\ 

\hline
 \textsc{Log. OC-DO}  &
\multicolumn{1}{c|}{16.7} & \multicolumn{1}{c|}{ 323 (84\%) } & \multicolumn{1}{c|}{0.0, 6.4, 1.6} &
\multicolumn{1}{c|}{9.0} & \multicolumn{1}{c|}{ 60 (81\%) } & \multicolumn{1}{c|}{0.0, 4.2, 0.0} &
\multicolumn{1}{c|}{18.3} & \multicolumn{1}{c|}{ 286 (88\%) } & \multicolumn{1}{c||}{0.0, 6.1, 2.5} &
\multicolumn{1}{c|}{13.3} & \multicolumn{1}{c|}{ 227 (77\%) } & \multicolumn{1}{c|}{0.0, 5.0, 0.9} &
\multicolumn{1}{c|}{10.3} & \multicolumn{1}{c|}{ 126 (72\%) } & \multicolumn{1}{c|}{0.0, 4.2, 0.2} &
\multicolumn{1}{c|}{12.7} & \multicolumn{1}{c|}{ 166 (78\%) } & \multicolumn{1}{c}{0.0, 4.9, 0.8} \\
\hline
 \textsc{Log. OC-DC}  &
  \multicolumn{1}{c|}{14.7} & \multicolumn{1}{c|}{ 359 (80\%) } & \multicolumn{1}{c|}{1.0, 9.1, 1.5} &
  \multicolumn{1}{c|}{11.0} & \multicolumn{1}{c|}{ 161 (85\%) } & \multicolumn{1}{c|}{1.0, 7.0, 0.2} &
  \multicolumn{1}{c|}{19.3} & \multicolumn{1}{c|}{ 377 (87\%) } & \multicolumn{1}{c||}{1.0, 10.9, 2.1} &
  \multicolumn{1}{c|}{12.3} & \multicolumn{1}{c|}{ 416 (75\%) } & \multicolumn{1}{c|}{1.0, 8.6, 1.9} &
  \multicolumn{1}{c|}{8.0} & \multicolumn{1}{c|}{ 251 (75\%) } & \multicolumn{1}{c|}{1.0, 6.5, 0.0} &
  \multicolumn{1}{c|}{10.0} & \multicolumn{1}{c|}{ 203 (78\%) } & \multicolumn{1}{c}{1.0, 6.9, 0.3} \\
\hline
 \textsc{Log. ALL-DO}  &
  \multicolumn{1}{c|}{14.0} & \multicolumn{1}{c|}{ 447 (72\%) } & \multicolumn{1}{c|}{0.0, 10.5, 2.1} &
  \multicolumn{1}{c|}{9.3} & \multicolumn{1}{c|}{ 96 (79\%) } & \multicolumn{1}{c|}{0.0, 6.4, 0.1} &
  \multicolumn{1}{c|}{17.0} & \multicolumn{1}{c|}{ 270 (87\%) } & \multicolumn{1}{c||}{0.0, 10.5, 2.6} &
  \multicolumn{1}{c|}{9.7} & \multicolumn{1}{c|}{ 349 (76\%) } & \multicolumn{1}{c|}{0.0, 8.7, 3.7} &
  \multicolumn{1}{c|}{7.3} & \multicolumn{1}{c|}{ 198 (81\%) } & \multicolumn{1}{c|}{0.0, 7.3, 3.1} &
  \multicolumn{1}{c|}{8.3} & \multicolumn{1}{c|}{ 195 (79\%) } & \multicolumn{1}{c}{0.0, 6.8, 1.8} \\
\hline
 \textsc{Log. ALL-DC}  &
  \multicolumn{1}{c|}{11.7} & \multicolumn{1}{c|}{ 401 (68\%) } & \multicolumn{1}{c|}{1.0, 12.3, 1.1} &
  \multicolumn{1}{c|}{10.3} & \multicolumn{1}{c|}{ 166 (81\%) } & \multicolumn{1}{c|}{0.9, 7.2, 0.4} &
  \multicolumn{1}{c|}{16.0} & \multicolumn{1}{c|}{ 392 (81\%) } & \multicolumn{1}{c||}{1.0, 11.9, 2.5} &
  \multicolumn{1}{c|}{7.3} & \multicolumn{1}{c|}{ 336 (76\%) } & \multicolumn{1}{c|}{1.0, 12.4, 3.1} &
  \multicolumn{1}{c|}{2.3} & \multicolumn{1}{c|}{ 165 (\textbf{92}\%) } & \multicolumn{1}{c|}{1.0, 3.6, 2.7} &
  \multicolumn{1}{c|}{6.0} & \multicolumn{1}{c|}{ 224 (76\%) } & \multicolumn{1}{c}{1.0, 10.3, 0.6} \\
\hline
 \textsc{JSP}  &
  \multicolumn{1}{c|}{13.0} & \multicolumn{1}{c|}{ 355 (76\%) } & \multicolumn{1}{c|}{1.5, 6.7, 0.3} &
  \multicolumn{1}{c|}{12.0} & \multicolumn{1}{c|}{ 152 (89\%) } & \multicolumn{1}{c|}{1.5, 3.6, 0.1} &
  \multicolumn{1}{c|}{17.0} & \multicolumn{1}{c|}{ 389 (91\%) } & \multicolumn{1}{c||}{1.9, 9.5, 0.5} &
  \multicolumn{1}{c|}{11.7} & \multicolumn{1}{c|}{ 342 (65\%) } & \multicolumn{1}{c|}{1.4, 3.9, 0.1} &
  \multicolumn{1}{c|}{12.3} & \multicolumn{1}{c|}{ 210 (77\%) } & \multicolumn{1}{c|}{1.5, 3.7, 0.0} &
  \multicolumn{1}{c|}{17.7} & \multicolumn{1}{c|}{ 291 (79\%) } & \multicolumn{1}{c}{2.0, 6.5, 0.1} \\
\hhline{=||=|=|=|=|=|=|=|=|=||=|=|=|=|=|=|=|=|=}
\textsc{\textbf{Total}} &
  \multicolumn{1}{c|}{70.0} & \multicolumn{1}{c|}{ 378 (76\%) } & \multicolumn{1}{c|}{0.7, 8.8, 1.4} &
  \multicolumn{1}{c|}{51.7} & \multicolumn{1}{c|}{ 132 (83\%) } & \multicolumn{1}{c|}{0.8, 5.7, 0.2} &
  \multicolumn{1}{c|}{\textbf{87.7}} & \multicolumn{1}{c|}{ 343 (87\%) } & \multicolumn{1}{c||}{0.8, 9.7, 2.1} &
  \multicolumn{1}{c|}{54.3} & \multicolumn{1}{c|}{ 332 (74\%) } & \multicolumn{1}{c|}{0.7, 7.2, 1.7} &
  \multicolumn{1}{c|}{40.3} & \multicolumn{1}{c|}{ 194 (77\%) } & \multicolumn{1}{c|}{0.7, 5.1, 0.8} &
  \multicolumn{1}{c|}{54.7} & \multicolumn{1}{c|}{ 227 (78\%) } & \multicolumn{1}{c}{0.9, 6.7, 0.6} \\

\end{tabular}%
}
\caption{Overall performance: each cell shows the number of problems solved (\# \textit{sol}), average planning time in seconds ($t$), percentage of time spent in motion planning ($\% t_p$), and average refinement counts (\textit{refs}) by type and layer: geometric (single activity), temporal (single activity), and group (combined geometric and temporal). All averaged over three runs per instance.}
\label{table:performance}
\end{table*}

We now evaluate our framework's ability to generate valid plans in complex multi-object scenarios using state-of-the-art solvers. We extend the logistics benchmark and the Job Shop Problem with transportation (JSP)~\cite{10.1007/978-3-319-33625-1_1} to include navigation tasks, stressing both the scheduler and motion planner with space-time refinements:

\begin{itemize}
    \item \textbf{Logistics}: $n_r$ robots, starting at a depot, must transport items from $n_s$ shelves back to the depot (as in Figure~\ref{fig:example1}). Shelves are arranged into narrow corridors, sometimes blocked by obstacles (closed doors) that must be moved to allow access. Each shelf contains $n_i$ items, accessible from either the corridor (inner) side or the outer side.

    \item \textbf{JSP}: $n_r$ robots must move $n_i$ items between $n_m$ machines for treatment. Then, each item is placed on a pallet for collection. The machines are initially blocked by closed doors. 
\end{itemize}
\color{black}
Available activities include robot \textit{navigation}, door \textit{opening}/\textit{closing}, and item \textit{loading}/\textit{unloading}, all with certain durations. 
Only navigation requires motion planning with obstacle avoidance, door activities are instantaneous changes of door configurations, and others are symbolic.
The optimization metric aims to minimize the makespan.

\myparagraph{Tests} As the first SAMP study, we adopt a 2D setup to establish the foundations. Despite its apparent simplicity, the problem remain challenging: multi-robot coordination requires time-parametrized, dynamically feasible trajectories in continuous space (with car-like dynamics), and the motion planner's search space grows exponentially with the number of agents. We consider the following test cases:

\begin{itemize}
    \item \textbf{Logistics}. We analyze $n_r \in \{1, 2, 3\}$ robots and $n_s = 2$ shelves forming a narrow corridor with an entrance door, each shelf holds $n_i = 4$ items [tot. instances: 24]. The door is open (DO in Table~\ref{table:performance}) or closed (DC), and items are picked only from the corridor (OC) or from both sides (ALL).
    
    \item \textbf{JSP}. We consider $n_r \in \{1, 2, 3\}$ robots, $n_i \in \{1, 2, 3\}$ items to be treated, and $n_m \in \{1, 2, 4, 6\}$ machines for treatment (i.e., $n_m$ doors to open) [tot. instances: 36]. 
\end{itemize}

Our framework is domain-independent and built upon the Unified Planning library~\cite{unified_planning_softwarex2025}, enabling seamless substitution or extension of the scheduling methods. On the motion planning side, it integrates the Open Motion Planning Library (OMPL)~\cite{sucan2012the-open-motion-planning-library}, supporting all its planners.
In our experiments, we use Aries~\cite{bitmonnot:hal-04174800} and its optimal variant Aries-opt (both with and without fluents), and our OR-Tools-based Constraint Programming Scheduling Engine (CPSE, without fluents). They are combined with RRT~\cite{LaValle1998RapidlyexploringRT} for path planning (Layer 1) and ST-RRT*~\cite{22-grothe-ICRA} for space-time multi-robot motion planning ($t_p = 10$ s, Layer 2), following the layering approach of the gray box of Algorithm~\ref{alg1:meta-engine}. We also instrument the collision checker to record obstacles encountered during the search. 

Layer 2 checks motion feasibility sequentially: for parallel robots, spatio-temporal trajectories are planned one at a time, each respecting previously planned trajectories to avoid collisions. This instantiation trades completeness for scalability, and remains aligned with the ST-RRT* setup~\cite{22-grothe-ICRA, 10.1609/icaps.v35i1.36120}, where scalability has been proved for up to 11 concurrent robots. This highlights the difficulty of our setting, which combines an already challenging multi-robot motion-planning problem with a scheduling problem that must sequence many (potentially optional) activities.

Tests were run on an AMD EPYC 7413 with a 1800 s timeout and a 20 GB memory limit.

\myparagraph{Results}
Table~\ref{table:performance} shows the performance averaged over three runs per instance, accounting for variability of sampling-based motion planners.
In both domains, all solvers solve at least one instance with 3 robots, confirming correct handling of temporal constraints and inter-object synchronization. Scheduling requires many refinement loops (avg. 9.1) due to optional navigation activities, while motion planning remains costly, taking up to 92\% of total planning time.

The framework effectively handles geometric complexity: in logistics scenarios, performance is similar whether doors are open or closed, showing robustness to spatial constraints.
On the temporal side, solving instances with up to 3 robots indicates that the framework can manage synchronization.
To further evaluate this ability and the benefits of parallelizing multi-robot activities, we compared the makespan of solutions produced by our approach with fully sequential schedules (i.e., no parallelization). For the instances considered, the average theoretical maximum improvement in makespan due to parallelization is 50\% (e.g., 2 robots picking 2 items from shelves in parallel versus sequentially). In all cases where parallelization can improve the makespan, our approach achieves an average reduction of 41\%.

Handling synchronization justifies the use of ST-RRT*, a typically expensive motion planner that, however, accounts for kinodynamic constraints and produces time-optimal trajectories. Planning times remain relatively low due to the layered architecture, which absorbs most geometric and temporal refinements at the single-action level, reducing multi-robot ST-RRT* calls (\textit{refs} column: single-action geometric refinements, single-action temporal refinements, joint geometric–temporal refinements at the motion-parallel-group level).
Layering also improves coverage: with both layers, 359 instances are solved on average; disabling Layer 1 reduces coverage to 140, and disabling only its single-robot temporal check (keeping the geometric one) yields 182.
To further assess the value of refinements, we evaluated a sequential pipeline: first solve scheduling in a motion-agnostic way, then invoke motion planning once, without refinement if it fails. In our setup, such sequential pipeline cannot solve any problem. Although some instances require no geometric refinements at the single-action level, there is always at least one temporal refinement. This is not due to unrealistic duration estimates: we use a standard symmetric trapezoidal velocity profile, but it ignores multi-agent interactions, necessitating additional temporal refinements.

Focusing on performance, Aries with fluents performs best, solving 87.7 instances. In logistics (see Figure~\ref{fig:samp_example}), at least one instance is solved with 1 robot and 8 items, 2 robots and 8 items, and 3 robots and 7 items; in JSP, instances are solved with 1–3 robots, up to 2 machines, and 3 pallets. Using fluents consistently improves performance, showing that richer state representations better guide refinements. Solvers generally handle more instances without makespan minimization, reflecting the added complexity of optimization. In Aries-Opt with fluents, the planning time spent on motion planning drops to 78\%, compared to 87\% in the non-optimal version, indicating more effort devoted to optimization. Without fluents, CPSE outperforms Aries (124.3 vs. 92.0 total instances solved, opt and non-opt), suggesting CPSE is more effective in this configuration. 
A final observation concerns refinements under different door settings. In the logistics domain, Aries-Opt without fluents generates more refinements when doors are open (ALL-DO) than when closed (ALL-DC). With open doors, fewer geometric bottlenecks and ordering constraints exist: the planner generates highly parallel schedules early in the search. Although symbolically consistent, these schedules can cause spatio-temporal conflicts at the motion-planning level (e.g., corridor congestion), requiring additional temporal refinements. When doors are closed, door-opening activities introduce explicit ordering and synchronization constraints that restrict parallelism and reduce invalid combinations.

\section{Conclusion and Future Work}
This paper defines the SAMP problem and presents a framework that solves it by interleaving off-the-shelf schedulers with (instrumented) motion planners, guided by incremental learning-based motion abstractions. The scheduler proposes candidate plans, and the motion planner checks feasibility, returning symbolic constraints to refine spatial and temporal decisions when needed.
Experiments on scheduling benchmarks with navigation tasks, testing various scheduling strategies (optimal, non-optimal, with/without fluents) and planners, show the framework’s effectiveness in handling multiple synchronized agents, coordinated stop-and-go behaviors, and complex spatio-temporal constraints.

In future work, we plan to extend our framework to support MAPF in addition to motion planning. Once we understand how to formulate this new problem and generate refinements, we will layer scheduling on top of MAPF to obtain a MAPF-aware scheduler, enabling the framework to tackle this problem as well and bridging continuous and discrete reasoning.

\section*{Acknowledgments}
This work has been partially supported by the AI4Work project funded by the EU Horizon 2020 research and innovation program under GA n. 101135990, the STEP-RL project funded by the European Research Council under GA n. 101115870, and by the Interconnected Nord-Est Innovation Ecosystem (iNEST) funded by the European Union Next-GenerationEU (Piano Nazionale di Ripresa e Resilienza (PNRR) – mission 4 component 2, investment 1.5 – D.D. 1058 23/06/2022, ECS00000043).

The work of Arthur Bit-Monnot has been supported by the HumFleet project ANR-23-CE33-0003 
and benefited from the AI Interdisciplinary Institute ANITI. ANITI is funded by the France 2030 program under the Grant agreement n°ANR-23-IACL-0002.

\bibliography{aaai2026}

\ifextendedversion
\clearpage
    \section{Appendix}
    This appendix contains the formal proofs of the guarantees (\textit{soundness}, \textit{completeness}, and \textit{optimality}) underlying the proposed framework, providing theoretical support for the properties discussed in the main text.

\subsection{Soundness}
The goal is to prove that if our framework returns a solution, then this solution is guaranteed to be correct, meaning that all scheduling constraints are satisfied, and the motion planner has found valid trajectories for all motion activities.

\begin{theorem}[Soundness]\label{thm:soundness}
    Let $\psi$ be a SAMP problem, if $\textsc{Solve}(\psi, opt, t_p, timeout)$ produces a solution $\pi$, then $\pi$ is valid.
\end{theorem}
\begin{proof}
Sketch. Let $\pi = \langle p, s, e, \tau \rangle$ and let $\rho = \langle p, s, e \rangle$. We need to prove that $\pi$ is non-conflicting and satisfies the condition of Definition 9.

As noted in the main paper, we model movable objects as unary resources, so since $\rho$ is a solution for the OS problem $\phi$ (because it is generated by the scheduler on the OS problem itself), it follows that $\pi$ is non-conflicting. Condition 3 of Definition 9 is satisfied by the scheduling problem posted to the scheduler, while conditions 1 and 2 are ensured because Algorithm 1 only returns at line 16, where the motion planner has successfully found a valid trajectory for all parallel motion groups.
\end{proof}

\subsection{Completeness}
The goal is to prove that our framework returns a solution whenever one exists, assuming the motion planner is complete and produces time-optimal solutions. This proof relies on showing that the learned constraints always cut the current candidate schedule from the solution space of the scheduler, but never cut valid solutions.

\begin{lemma}[Geometric Progression]\label{thm:geometric-progression}
    Suppose a geometric refinement is derived out of a candidate schedule $\rho$ for the SAMP problem $\psi$ with parallel motion group $\mathcal{G}$, introducing conflicts $\Sigma$ and $\Omega$ and yielding a refined SAMP problem $\psi'$. Then, $\rho$ is no longer a candidate schedule for $\psi'$.
\end{lemma}
\begin{proof}
Sketch. The geometric refinement added to $\psi'$ is the constraint:
$$
\small
\textsc{RCond}(\mathcal{G}) \! \rightarrow \!\!\!\!\!\! \bigvee_{b \in \mathcal{G}, o \in \omega_b} \!\!\!\!\!\! \textsc{ChConf}(b, \textit{conf}(o))
$$
To show that $\rho$ is not a candidate schedule for $\psi'$, it suffices to show that $\rho$ violates this constraint. Clearly, $\textsc{RCond}(\mathcal{G})$ is satisfied by $\rho$: all activities in $\mathcal{G}$ exist in $\rho$, the order of activities is the same (as it only depends on $\rho$) and rivals are either before or after the activities in $\mathcal{G}$, otherwise they would have been part of $\mathcal{G}$ in the first place.

Instead, the formula $\bigvee_{b \in \mathcal{G}, o \in \omega_b} \textsc{ChConf}(b, \textit{conf}(o))$ is violated because the configuration of every $o \in \mathcal{O}$ only depends on $\rho$, and thus all $\textsc{ChConf}(b, \textit{conf}(o))$ are false for any $b$ and any $o$.
\end{proof}

\begin{lemma}[Temporal Progression]\label{thm:temporal-progression}
Suppose a temporal refinement is derived out of a candidate schedule $\rho$ for the SAMP problem $\psi$ with parallel motion group $\mathcal{G}$, resulting in a new SAMP problem $\psi'$. Then $\rho$ is not a candidate schedule for $\psi'$. 
\end{lemma}
\begin{proof}
Sketch. The temporal refinement added to $\psi'$ is the constraint:
\begin{align*}
\small
\textsc{RCond}(\mathcal{G}) \rightarrow \textsc{ChTime}(\overline{\mathcal{G}})
\end{align*}
with $\overline{\mathcal{G}} \subseteq \mathcal{G}$ and $\textsc{ChTime}(\overline{\mathcal{G}})$ equals to
\begin{equation*}
\small
\begin{split}
\bigvee_{\substack{a\in\mathcal{G}\\ o \in \mathcal{O}}} \!\! \textsc{ChConf}(a, \textit{conf}(o)) \vee
\bigvee_{a \in \overline{\mathcal{G}}} a.\text{start} - \min_{b\in \overline{\mathcal{G}}} (b.\text{start}) < \delta_a \vee \\
\bigvee_{a\in \overline{\mathcal{G}} | (d_a+\delta_a) < (\overline{\delta}_a + \overline{d}_a)} (a.\text{end} - \min_{b\in \overline{\mathcal{G}}} (b.\text{start})) \geq \overline{\delta}_a + \overline{d}_a
\end{split}
\end{equation*}
To show that $\rho$ is not a candidate schedule for $\psi'$, it suffices to show that $\rho$ violates this constraint. 
$\textsc{RCond}(\mathcal{G})$ is satisfied by $\rho$ as in the previous proof.

The formula $\bigvee_{a\in\mathcal{G}, o \in \mathcal{O}} \textsc{ChConf}(a, \textit{conf}(o))$ is violated because the configuration of each $o \in \mathcal{O}$ only depends on $\rho$, and thus all $\textsc{ChConf}(b, \textit{conf}(o))$ are false for any $b$ and $o$.

The formula $\bigvee_{a \in \overline{\mathcal{G}}} a.\text{start} - \min_{b\in \overline{\mathcal{G}}} (b.\text{start}) < \delta_a$ is violated because in $\rho$, for any $a \in \overline{\mathcal{G}}$, $s(a) - \min_{b\in \overline{\mathcal{G}}} (s(b)) = \delta_a$ by definition.

$\bigvee_{a\in \overline{\mathcal{G}} | (d_a+\delta_a) < (\overline{\delta}_a + \overline{d}_a)} (a.\text{end} - \min_{b\in \overline{\mathcal{G}}} (b.\text{start})) \geq \overline{\delta}_a + \overline{d}_a$ is violated because each activity $a$ is such that $e(a) - \min_{b\in \overline{\mathcal{G}}} (s(b)) = d_a + \delta_a$ by definition. Hence, all disjuncts are trivially false.
\end{proof}

\begin{lemma}[Geometric pruning soundness]\label{thm:geometric-pruning}
Suppose a geometric refinement is derived out of a candidate schedule $\rho$ for $\psi$ with parallel motion group $\mathcal{G}$ and conflict $\Sigma$ and $\Omega$, resulting in SAMP $\psi'$. Any SAMP solution $\pi$ of $\psi$ is a SAMP solution for $\psi'$. 
\end{lemma}
\begin{proof}
Sketch. Suppose, for the sake of contradiction, that there exists a solution $\pi = \langle p, s, e, \tau \rangle$ for $\psi$ which is not a solution for $\psi'$. Let $\rho'$ be the OS schedule of $\pi$. Since the only difference between $\psi$ and $\psi'$ is the constraint
$$
\small
\textsc{RCond}(\mathcal{G}) \! \rightarrow \!\!\!\!\!\! \bigvee_{b \in \mathcal{G}, o \in \omega_b} \!\!\!\!\!\! \textsc{ChConf}(b, \textit{conf}(o))
$$
then, $\pi$ must violate this constraint. 

To violate this constraint, $\pi$ must satisfy $\textsc{RCond}(\mathcal{G})$ and violate $\bigvee_{b \in \mathcal{G}, o \in \omega_b} \textsc{ChConf}(b, \textit{conf}(o))$. 

Hence, $\pi$ has the set of activities $\mathcal{G}$, all present and such that any other motion activity is either before the start of the first activity in $\mathcal{G}$ or after the last end. This is because of the first and third conjuncts of $\textsc{RCond}(\mathcal{G})$.

Note that \textit{conf(o)} in $\rho$ must be equal to \textit{conf(o)} in $\rho'$ for all $b \in \mathcal{G}$ and $o \in \omega_b$.
Moreover, $\textit{conf}(o_a) = \textit{conf}'(o_a)$ for all $a \in \mathcal{G}$, because of the second conjunct of $\textsc{RCond}(\mathcal{G})$: the first motion action for any moved object is kept, therefore every movable object moved by $\mathcal{G}$ is initially in the same configuration.
Thus, all relevant obstacles (we assume that the motion planner is complete, therefore all obstacles that can be encountered by any movable object are returned) and movable objects are in the same configurations in $\rho$ and $\rho'$. 

Now, consider the paths for motion activities in $\mathcal{G}$ that must exist for $\pi$ to be a solution. Observe that, changing the order of the activities (after the first one), moving $o_a$ cannot change the motion planner verdict on the feasibility of the combined motion of the activities in $\mathcal{G}$. This is because in geometric refinements we are not considering the timings and the problem constraints require the sequence of activities moving $o_a$ to be such that the ending configuration of one activity is the initial configuration of the following one. The only thing that matters for the path existence is whether the paths are geometrically realizable, and this is unaffected by changing the order of waypoints to be reached by one movable object (but without constraining the order of waypoints between different movable objects).

However, we know that the motion planner instantiated on the candidate schedule $\rho$ deemed the problem unsolvable. This leads to a contradiction, as $\pi$ cannot be a solution for $\psi$ under these conditions.
\end{proof}

\begin{lemma}[Temporal pruning soundness]\label{thm:temporal-pruning}
Suppose a temporal refinement is derived out of a candidate schedule $\rho$ for $\psi$ with parallel motion group $\mathcal{G}$, resulting in a new SAMP $\psi'$. Any SAMP solution $\pi$ of $\psi$ is a SAMP solution for $\psi'$. 
\end{lemma}
\begin{proof}
Sketch. We proceed as in the previous lemma. Suppose, for the sake of contradiction, that there exists a solution $\pi = \langle p, s, e, \tau \rangle$ for $\psi$ which is not a solution for $\psi'$. Let $\rho'$ be the OS schedule of $\pi$. Since the only difference between $\psi$ and $\psi'$ is the constraint
\begin{align*}
\small
\textsc{RCond}(\mathcal{G}) \rightarrow \textsc{ChTime}(\overline{\mathcal{G}})
\end{align*}
with $\overline{\mathcal{G}} \subseteq \mathcal{G}$ and $\textsc{ChTime}(\overline{\mathcal{G}})$ equals to
\begin{equation*}
\small
\begin{split}
\bigvee_{\substack{a\in\mathcal{G}\\ o \in \mathcal{O}}} \!\! \textsc{ChConf}(a, \textit{conf}(o)) \vee
\bigvee_{a \in \overline{\mathcal{G}}} a.\text{start} - \min_{b\in \overline{\mathcal{G}}} (b.\text{start}) < \delta_a \vee \\
\bigvee_{a\in \overline{\mathcal{G}} | (d_a+\delta_a) < (\overline{\delta}_a + \overline{d}_a)} (a.\text{end} - \min_{b\in \overline{\mathcal{G}}} (b.\text{start})) \geq \overline{\delta}_a + \overline{d}_a
\end{split}
\end{equation*}
then, $\pi$ must violate this constraint. 

To violate this constraint, $\pi$ must satisfy $\textsc{RCond}(\mathcal{G})$ and violate $\textsc{ChTime}(\overline{\mathcal{G}})$. As before, $\pi$ has the set of activities $\mathcal{G}$, all present and such that all other motion activity is either before the start of the first activity in $\mathcal{G}$ or after the last end. Moreover, \textit{conf(o)} in $\rho$ must be equal to \textit{conf(o)} in $\rho'$ for all $o \in \mathcal{O}$. Additionally, for all $a \in \overline{\mathcal{G}}$ it must hold $s(a) - \min_{b\in \overline{\mathcal{G}}} (s(b)) \ge \delta_a$ and  $e(a) - \min_{b\in \overline{\mathcal{G}}} (s(b)) < \overline{\delta}_a + \overline{d}_a$.

But then consider the trajectories $\tau(a)$ for all $a \in \overline{\mathcal{G}}$, these trajectories are such that each movable object $o_a$ is stationary in the interval $[\min_{b\in \overline{\mathcal{G}}} (s(b)), s(a)]$ and the movement associated with $a$ ends before $e(a)$. Therefore, $\tau$ witnesses a solution for a motion planning problem that is at least as constrained (from the temporal point of view) as the one derived from $\rho$, which is strictly faster than the one used to generate the temporal refinement. But we assumed the motion planner was optimal and complete, hence the contradiction.
\end{proof}

\begin{theorem}[Completeness]\label{thm:completeness}
Let $\psi$ be SAMP problem admitting at least one solution. $\textsc{Solve}(\psi, opt, t_p, \infty)$ eventually returns a solution for $\psi$ assuming a motion planner which is complete and optimal.
\end{theorem}
\begin{proof}
Sketch. The theorem follows from four observations. First, no candidate schedule is evaluated twice by the algorithm because of Lemmas \ref{thm:geometric-progression} and \ref{thm:temporal-progression}. Second, no solution is cut from the solution space because of Lemmas \ref{thm:geometric-pruning} and \ref{thm:temporal-pruning}. Third, the set of candidate plans for a SAMP problem is finite (even though we formalized time over the natural numbers, there is an obvious time horizon defined by the sum of all the maximal durations ov every activity). Fourth, assuming the motion planner is complete and optimal, eventually we will arrive at a time bound $t_p$ sufficient to construct the trajectories for the solution. Therefore, a solution is eventually found if it exists.
\end{proof}

Finally, note that the approach is \emph{not} a decision procedure, if no solution exists for the SAMP problem the algorithm might diverge.

\subsection{Optimality}

The goal is to prove that our framework is relatively optimal for makespan optimization.

\begin{theorem}[Relative Optimality]
Assuming $opt$ is the function aiming to minimize the makespan of the schedule, a solution $\pi$ returned by $\textsc{Solve}(\psi, opt, t_p, timeout)$ is optimal, assuming the motion planner is complete and optimal w.r.t. the duration of motions.
\end{theorem}
\begin{proof}
For the sake of contradiction, assume there exists a solution $\pi'$ with a makespan smaller than $\pi$. Since we assume that the scheduler is optimal, complete and correct, the candidate schedule $\rho'$ of $\pi'$ would be encountered before the candidate schedule $\rho$ of $\pi$. Because of Lemmas~\ref{thm:geometric-pruning} and \ref{thm:temporal-pruning} we know that no valid solution is discarded, therefore $\pi'$ would be returned instead of $\pi$, leading to the contradiction.
\end{proof}
\fi

\end{document}